\documentclass[letterpaper]{article} 
\usepackage{aaai23}  
\usepackage{times}  
\usepackage{helvet}  
\usepackage{courier}  
\usepackage[hyphens]{url}  
\usepackage{graphicx} 
\usepackage{natbib}  
\usepackage{caption} 
\usepackage{comment}
\usepackage{algorithm}
\usepackage{algorithmic}

\usepackage{newfloat}
\usepackage{listings}
\DeclareCaptionStyle{ruled}{labelfont=normalfont,labelsep=colon,strut=off} 
\floatstyle{ruled}
\newfloat{listing}{tb}{lst}{}
\floatname{listing}{Listing}
\title{Tight Performance Guarantees of Imitator Policies with Continuous Actions}
\author{
    Davide Maran, Alberto Maria Metelli, Marcello Restelli
}
\affiliations{
    Politecnico di Milano,\\

    Piazza Leonardo da Vinci, 32 20133, Milan, Italy\\
    \{davide.maran, albertomaria.metelli, macrello.restelli\}@polimi.it
}

\usepackage{bibentry}

\RequirePackage{mathtools}
\RequirePackage{amsthm}  
\RequirePackage{amsmath}
\RequirePackage{amssymb}
\RequirePackage{bm}
\RequirePackage{mathabx}
\RequirePackage{xspace}
\RequirePackage{twoopt}
\RequirePackage{mathrsfs} 
\RequirePackage{xifthen}
\RequirePackage{etoolbox}
\RequirePackage{dsfont}
\RequirePackage{ifthen}

\usepackage{enumitem}
\usepackage{thmtools}

\usepackage{thmtools}
\usepackage{thm-restate}

\newtheorem{defin}{Definition}

\newtheorem{thm}{Theorem}

\newtheorem{lem}{Lemma}
\newtheorem{cor}[thm]{Corollary}
\newtheorem{prop}[thm]{Proposition}
\newtheorem{exa}{Example}
\newtheorem{ass}{Assumption}
\newtheorem{remark}{Remark}

\newcommand{\R}{\mathbb{R}}
\newcommand{\N}{\mathbb{N}}
\newcommand{\E}{\mathop{\mathbb{E}}}

\newcommand{\wass}{\mathcal W}

\makeatletter
\newcommand*{\bdiv}{%
  \nonscript\mskip-\medmuskip\mkern5mu%
  \mathbin{\operator@font div}\penalty900\mkern5mu%
  \nonscript\mskip-\medmuskip
}
\makeatother

\DeclareMathOperator*{\Proba}{\mathbb{P}}

\newcommand{\MDP}{\mathcal{M}}
\newcommand{\Ss}{\mathcal{S}}
\newcommand{\As}{\mathcal{A}}
\newcommand{\Xs}{\mathcal{X}}
\newcommand{\SAs}{\mathcal{S}\times\mathcal{A}}

\newcommand{\Ys}{\mathcal{Y}}
\newcommand{\PM}[1]{\mathcal{P}{(#1)}}

\newcommand{\tv}{\mathrm{TV}}
\newcommand{\de}{\mathrm{d}}

\newcommand{\diam}[1]{\mathrm{diam}{(#1)}}

\DeclareRobustCommand{\eg}{e.g.,\@\xspace}                         
\DeclareRobustCommand{\ie}{i.e.,\@\xspace}                         
\DeclareRobustCommand{\wrt}{w.r.t.\@\xspace}

\usepackage[colorinlistoftodos]{todonotes}
\definecolor{citrine}{rgb}{0.89, 0.82, 0.04}
\definecolor{blued}{RGB}{70,197,221}
\definecolor{applegreen}{rgb}{0.55, 0.71, 0.0}
\definecolor{flame}{rgb}{0.89, 0.35, 0.13}

\renewcommand\thmcontinues[1]{\textbf{continued}}

\begin{document}

\maketitle

\begin{abstract}
Behavioral Cloning (BC) aims at learning a policy that mimics the behavior demonstrated by an expert. The current theoretical understanding of BC is limited to the case of finite actions. In this paper, we study BC with the goal of providing theoretical guarantees on the performance of the imitator policy in the case of continuous actions. We start by deriving a novel bound on the performance gap based on Wasserstein distance, applicable for continuous-action experts, holding under the assumption that the value function is Lipschitz continuous. Since this latter condition is hardy fulfilled in practice, even for Lipschitz Markov Decision Processes and policies, we propose a relaxed setting, proving that value function is always H\"older continuous. This result is of independent interest and allows obtaining in BC a general bound for the performance of the imitator policy. Finally, we analyze noise injection, a common practice in which the expert's action is executed in the environment after the application of a noise kernel. We show that this practice allows deriving stronger performance guarantees, at the price of a bias due to the noise addition.
\end{abstract}

\section{Introduction}\label{sec:intro}
The degree of interaction of the human in the ecosystem of artificial intelligence is progressively becoming more and more prominent~\cite{Zanzotto19}. In this setting, the human plays the role of an expert that, with different tools, interacts with the artificial agents and allows the agent to leverage their knowledge to improve, quicken, and make the learning process more effective~\cite{JeonMD20}.

Imitation Learning~\citep[IL,][]{OsaPNBA018} can be considered one of the simplest forms of interaction between a human and an artificial agent. This kind of interaction is unidirectional since the human expert provides the agent with a set of demonstrations of behavior that is optimal w.r.t. an unknown objective. The agent, on its part, aims to learn a behavior as close as possible to the demonstrated one. Classically, we distinguish between two realizations of IL: Behavioral Cloning~\citep[BC,][]{BainS95} and Inverse Reinforcement Learning~\citep[IRL,][]{AroraD21}. BC aims at mimicking the \emph{behavior} of the agent by recovering a policy that matches as much as possible the expert's demonstrated behavior. Instead, IRL has the more ambitious goal of reconstructing a \emph{reward} function that justifies the expert's behavior. Thus, it aims at representing the expert's \emph{intent} rather than their behavior. In this sense, IRL is more challenging than BC, as its output, the reward function, is a more powerful tool that succeeds in being deployed even in the presence of a modification of the environment.

Although IL techniques have been successfully applied to a large variety of real-world applications~\citep[\eg][]{AsfourAGD08,geng2011transferring,RozoJT13,LikmetaMRTGR21}, their theoretical understanding in terms of performance of the imitation policy is currently limited. Recently, in~\cite{il:bounds}, a first analysis of the error bounds has been provided for BC and Generative Adversarial Imitation Learning~\cite{HoE16}. However, these results involve the presence of an $f$-divergence~\cite{renyi1961measures}, usually total variation (TV) or KL-divergence, between the expert's policy and the imitator one. Consequently, they are significant only when the action space is finite, while becoming vacuous for experts with continuous actions. To further argue on the limitation of this analysis, consider the case in which BC is reduced to minimize the \emph{mean squared error} (MSE) between the expert's action and the imitator one. Even in this simple scenario, as we shall see, the current analysis based on TV cannot relate MSE with the performance of the imitator policy. This represents a relevant limitation since many of the applications of IL are naturally defined with continuous-actions context.

\paragraph{Original Contributions} In this paper, we aim to take a step forward to a more comprehensive theoretical understanding of BC. Specifically, we devise error bounds that relate the performance difference $J^{\pi_E}-J^{\pi_I}$ between the expert's policy $\pi_E$ and the imitator one $\pi_I$ to their divergence. Our bounds are based on the Wasserstein distance~\cite{villani2009optimal} and, for this reason, are meaningful even in the presence of continuous-action spaces (Section~\ref{sec:firstBound}). Our work contains the following contributions:
\begin{enumerate}
    \item We will prove a performance bound for standard BC in case of Lipschitz reward-transition for the MDP (see ~\cite{lyp:locality}) and Lipschitz continuity of the value function.
    \item Since the latter assumption is often violated in practice\footnote{It is well-known that the value function is Lipschitz continuous under the demanding assumption that $\gamma L_p(1+L_{\pi}) < 1$~\cite{lyp:locality}, requiring the Lipschitz constants of the transition model $L_p$ and of the policy $L_{\pi}$ to be very small.}, we extend the result by only requiring Lipschitzness of the MDP, even if it requires a weaker performance bound. It is also shown that, the less regularity we lose on the value function, the slower the convergence of BC (Section~\ref{sec:holder}). 
    \item Finally, we focus on a popular practice employed in imitation learning, \ie \emph{noise injection}~\cite{LaskeyLFDG17}. In this setting, the expert's action, before being executed in the environment, is corrupted with noise to make the imitation process more robust. We show that noise injection allows achieving stronger theoretical guarantees at the price of competing against a noisy expert, which could have a lower performance (Section~\ref{sec:noiseInj}).
\end{enumerate}
In particular, in the second point, we show that the value function of a Lipschitz MDP is always H\"older continuous, with a suitable choice of the exponent depending on the properties of the MDP and policy. This represents a result of independent interest that overcomes a well-known limitation of the Lipschitz continuity of the value function~\cite{lyp:locality,PirottaRB15}, with possible applications outside BC.

\section{Preliminaries}
In this section, we provide the necessary mathematical background (Section~\ref{sec:prelimMath}) and the foundations of Markov Decision Processes (Section~\ref{sec:prelimMDP}).

\subsection{Mathematical Background}\label{sec:prelimMath}

\paragraph{Notation} Let $\Xs$ be a set and $\mathfrak{F}$ be a $\sigma$-algebra over $\Xs$, we denote with $\PM{\Xs}$ the set of probability measures over the measurable space $(\Xs,\mathfrak{F} )$. Let $x \in \Xs$, we denote the Dirac delta measure centered in $x$ as $\delta_x$. Let $f : \Xs \rightarrow \mathbb{R}$ be a function, we denote the $L_\infty$-norm as $\|f\|_{\infty} = \sup_{x \in \Xs} f(x)$ and with $\|f\|_i$ the $L_{i}$-norm for $i \in \{1,2\}$.

\paragraph{Lipschitz Continuity}
Let $(\Xs,d_{\Xs})$ and $(\Ys,d_{\Ys})$ be two metric spaces and $L > 0$. A function $f:\Xs \rightarrow \Ys$ is said to be $L$-\emph{Lipschitz continuous} ($L$-LC) if:
\begin{align*}
    d_{\Ys}(f(x),f(x'))\leq L d_{\Xs}(x,x'), \quad \forall x,x' \in \Xs.
\end{align*}
We denote the Lipschitz semi-norm of function $f$ as $\left\Vert f\right\Vert_L = \sup_{x,x'\in \Xs, x\neq x'} {d_{\Ys}(f(x),f(x'))}/{d_{\Xs}(x,x')}$. In the real space ($X\subseteq \mathbb{R}^n$), we use the Euclidean distance, i.e., $d_{\Xs}(x,x')=\| x-x'\|_2$. For probability measures ($\Xs = \PM{\Omega}$), the most intuitive distance is the \emph{total variation} (TV), defined as:
\begin{align*}
    \tv(\mu, \nu) = \sup_{\|f\|_\infty \le 1} \left| \int_{\Omega} f(\omega) \left( \mu - \nu \right) (\de \omega) \right| \; \forall \mu,\nu \in \PM{\Omega}
\end{align*}
However, with continuous deterministic distributions, the TV takes its maximum value $1$ (Figure~\ref{wass_dist}). Thus, we introduce $L_1$-\emph{Wasserstein} distance \citep{villani2009optimal}, defined as:
\begin{align*}
    \wass(\mu , \nu)=\sup_{\left\Vert f\right\Vert_L\leq 1} \left\vert \int_{\Omega} f(\omega)(\mu-\nu) (\de \omega) \right\vert \quad \forall \mu,\nu \in \PM{\Omega} 
\end{align*}
It is worth noting that, for deterministic distributions, we have $\wass(\delta_x , \delta_{x'}) = d_{\Xs}(x,x')$.

\paragraph{H\"older Continuity}
The notion of Lipschitz continuity is generalized by H\"older continuity. Let $(\Xs,d_{\Xs})$ and $(\Ys,d_{\Ys})$ be two metric spaces and $L,\alpha > 0$. A function $f:\Xs \rightarrow \Ys$ is said to be $(\alpha,L)$-\emph{H\"older continuous} ($(\alpha,L)$-HC) if:
\begin{align*}
    d_{\Ys}(f(x),f(x'))\leq L d_{\Xs}(x,x')^\alpha, \quad \forall x,x' \in \Xs.
\end{align*}
It is worth noting that: (i) Lipschitz continuity is obtained by H\"older continuity for $\alpha=1$; (ii) only constant functions are $(\alpha,L)-$HC for $\alpha>1$; (iii) in bounded domains, the higher the value of $\alpha$ the more restrictive the condition.

\paragraph{Convolution} Let $f,g: \mathbb R^n\to \mathbb R$ be two functions, their \emph{convolution} is defined for all $x \in \mathbb R^n$ as:
$$(f*g)(x) \coloneqq \int_{\mathbb R^n} f(x-y)g(y) \de y = \int_{\mathbb R^n} f(y)g(x-y) \de y.$$
We introduce the following regularity assumption regarding the probability measures.
\begin{defin}\label{tv_lips}
    A probability measure $\mathcal{L} \in \PM{\mathbb{R}^n}$ is $L$-TV-Lipschitz continuous ($L$-TV-LC) if:
    $$\tv(\mathcal{L}(\cdot + h),\mathcal{L}(\cdot))\le L\|h\|_2, \qquad \forall h\in \mathbb R^n.$$
\end{defin}
Under this assumption, we can prove that the convolution regularizes bounded and possibly irregular functions.
\begin{restatable}[]{prop}{smooth}\label{smooth} 
    Let $f: \mathbb{R}^n \rightarrow \mathbb{R}$ be a function such that $\|f\|_\infty \le M $, and let $\mathcal{L} \in \PM{\mathbb{R}^n}$ be an $L$-TV-LC probability measure that admits density function $\ell: \mathbb{R}^n \rightarrow \mathbb{R}_{\ge 0}$. Then, the convolution $f*\ell$ is $2LM$-LC continuous.
\end{restatable}

\subsection{Markov Decision Processes}\label{sec:prelimMDP}

A discrete-time discounted Markov Decision Process~\citep[MDP,][]{puterman2014markov} is a 6-tuple $\mathcal{M} = (\Ss,\As, p, r, \gamma, \mu)$ where $\Ss $ and $\As$ are the measurable sets of states and actions, $p: \SAs \rightarrow\PM{\Ss}$ is the transition model that defines the probability measure $p(\cdot \vert s,a)$ of the next state when playing action $a\in\As$ in state $s\in\Ss$, $r: \SAs \rightarrow \mathbb{R}$ is the reward function defining the reward $r(s,a)$ upon playing action $a\in\As$ in state $s\in\Ss$, $\gamma \in [0,1)$ is the discount factor, and $\mu \in \PM{\Ss}$ is the initial-state distribution.
The agent's behavior is modeled by a (Markovian stationary) policy $\pi: \Ss \rightarrow \PM{\As}$, which assigns a probability measure $\pi(\cdot|s)$ of the action to be taken in state $s\in\Ss$. With little abuse of notation, when the policy is deterministic, we denote with $\pi(s)$ the action played in state $s\in \Ss$. The execution of a policy determines a \emph{$\gamma$-discounted visitation distribution}, defined as: 
\begin{align*}
    d^\pi(s) \coloneqq (1-\gamma)\sum_{t=0}^{+\infty} \gamma^t\Proba(s_t=s\vert \pi, \mu), \quad \forall s \in \Ss.
\end{align*}

\paragraph{Value Functions}
The state-action value function (or \emph{Q-function}) which quantifies the expected discounted sum of the rewards obtained under a policy $\pi$, starting from a state $s\in\Ss$ and fixing the first action $a\in\As$:
\begin{align}\label{eq:action_state_val}
    Q^{\pi}(s,a) \coloneqq \mathbb{E}_{\pi} \left[ \sum_{t=0}^{+\infty} \gamma^t r(s_t,a_t)\bigg | s_0=s,a_0=a  \right],
\end{align}
where $\E_{\pi}$ denotes the expectation \wrt to the stochastic process $a_t \sim \pi(\cdot|s_t)$ and $s_{t+1} \sim p(\cdot|s_t,a_t)$ for all $t \in \mathbb{N}$. The state value function (or \emph{V-function}) is defined as $V^\pi(s)\coloneqq\E_{a\sim\pi(\cdot\vert s)}[Q^\pi(s,a)]$, for all $s\in\Ss$. Given an initial state distribution $\mu$, the \emph{expected return} is defined as:
$$J^{\pi} \coloneqq \E_{s \sim \mu}[V^{\pi}(s)] =  \frac{1}{1-\gamma} \E_{s \sim d^{\pi}, a \sim \pi(\cdot|s)}[r(s,a)].$$

\paragraph{Lipschitz MDPs}
We now introduce notions that will allow us to characterize the smoothness of an MDP~\cite{lyp:locality}. To this end, we assume that the state space $\Ss$ and action space $\As$ are metric spaces endowed with the corresponding distance functions $d_{\Ss}$ and $d_{\As}$.

\begin{ass}[Lipschitz MDP]\label{ass:LipMDP}
An MDP $\mathcal{M}$ is $(L_p,L_r)$-LC if, for all $(s,a),(s',a')\in\Ss\times\As$ it holds that:
\begin{align*}
    & \wass(p(\cdot\vert s,a), p(\cdot\vert s',a'))\leq L_p \left( d_{\Ss}(s,s')  + d_{\As}(a,a')\right),
    \\
    & \left\vert r(s,a)-r(s',a') \right\vert \leq L_r \left( d_{\Ss}(s,s') + d_{\As}(a,a')\right).
\end{align*}
\end{ass}
\begin{ass}[Lipschitz Policy]\label{ass:LipPol}
A (Markovian stationary) policy $\pi$ is $L_\pi$-LC if, for all $ s,s'\in\Ss$  it holds that:
\begin{align*}
    \wass(\pi(\cdot\vert s), \pi(\cdot\vert s'))\leq L_\pi d_{\Ss}(s,s').
\end{align*}
\end{ass}

\begin{figure}[t]
\centering
\includegraphics[width=\linewidth]{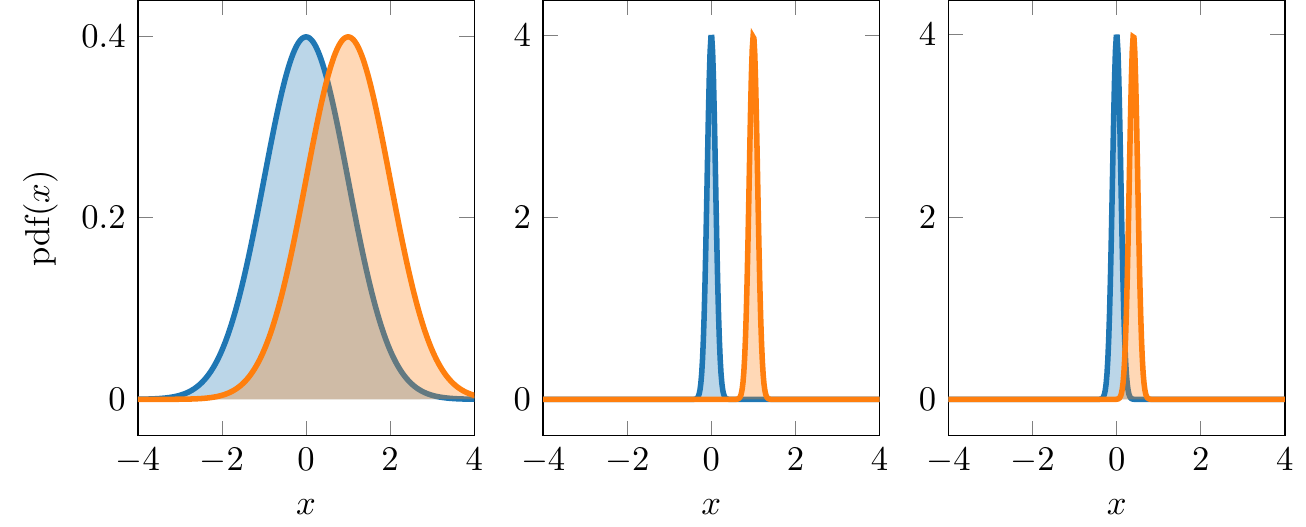} 
\caption{Comparison between TV and Wasserstein distances for two Gaussian distributions $\mu$ and $\nu$. Left: $\tv(\mu,\nu)\approx 0.38$, $\wass(\mu,\nu)=1$, Center: $\tv(\mu,\nu)\approx 1$, $\wass(\mu,\nu)=1$, Right: $\tv(\mu,\nu)\approx 1$, $\wass(\mu,\nu)=0.4$}
\label{wass_dist}
\end{figure}

Note, if instead of the Wasserstein metric, we had used the TV, these assumptions would be way more restrictive, not holding for deterministic environment/policies with continuous state-action spaces~\cite{munos2008finite}.
Under Assumptions~\ref{ass:LipMDP} and \ref{ass:LipPol},
provided that $\gamma L_p(1+L_\pi)< 1$, the Q-function $Q^\pi$ is $L_Q$-LC with $L_Q \le \frac{L_r}{1-\gamma L_p(1+L_\pi)}$ \citep[Theorem 1]{lyp:locality}.

\section{Bound for Imitating Policies based on Wasserstein Distance}\label{sec:firstBound}

The high-level goal of this work is to find a theoretical guarantee for the imitator policies learned with BC. Specifically, we want to bound the difference in expected return $J^{\pi_E} - J^{\pi_I}$ between the \emph{imitator} policy $\pi_I$ learned with BC and the \emph{expert} policy $\pi_E$ in terms of a distributional divergence between the corresponding action distributions.

The best-known results for this kind of analysis, in the case of \emph{discrete} action spaces, are proved in \cite{il:bounds} and we report it below for completeness.\footnote{The result reported in \cite{il:bounds} involves the KL-divergence and is obtained, via Pinsker's inequality, from the one we report that is tighter (Appendix A.2 of~\citet{il:bounds}).}

\begin{thm}[\citet{il:bounds}, Theorem 1]\label{thm:old}
Let $\pi_E$ be the expert policy and $\pi_I$ be the imitator policy. If $|r(s,a)| \le R_{\max}$ for all $(s,a) \in \SAs$, it holds that:
$$J^{\pi_E}- J^{\pi_I}\le \frac{2 R_{\max}}{(1-\gamma)^2}\E_{s\sim d^{\pi_E}} [\tv(\pi_E(\cdot|s), \pi_I(\cdot|s))].$$
\end{thm}

As anticipated, this result is not suitable for continuous action spaces, since the TV between different policies would take its maximum value $1$ whenever one of the two policies is deterministic. The following example clarifies the issue.

\begin{exa}[label=exa:1]
Suppose that the action space is a real space $\As \subseteq \mathbb{R}^n$ and that both expert  $\pi_E$ and the imitator  $\pi_I$ policies are deterministic. A common way to perform BC is to minimize the \emph{mean squared error} (MSE) between the expert's action and the imitator one. Suppose we are able to provide the following guarantee on the MSE, for some $\varepsilon>0$:
\begin{align}\label{eq:mse}
   \E_{s \sim d^E} \left[\left\| \pi_E(s) - \pi_I(s)\right\|^2_2 \right] \le \varepsilon^2.
\end{align}
However, this condition provides no guarantee in TV. Indeed, by taking $\pi_I(s) = \pi_E(s) + \frac{\varepsilon}{\sqrt{n}} \mathbf{1}_n$, being $\mathbf{1}_n$ the vector of all $1$s, Equation~\eqref{eq:mse} is fulfilled, but we obtain: $\E_{s\sim d^{\pi_E}} [\tv(\pi_E(\cdot|s), \pi_I(\cdot|s))] = \E_{s\sim d^{\pi_E}} [\mathds{1}\{\pi_E(s) \neq \pi_I(s)\}] = 1$, where $\mathds{1}$ is the indicator function.
\end{exa}

\subsection{A Bound based on Wasserstein Distance}
Even if the existing analysis of \citet{il:bounds} cannot be applied in continuous action spaces, as shown in Example~\ref{exa:1}, it is not hard to leverage the regularity of the MDP to effectively bound the performance difference $J^{\pi_E}-J^{\pi_I}$.

\begin{thm}\label{police}
    Let $\pi_E$ be the expert policy and $\pi_I$ be the imitator policy. If that state-action value function $Q^{\pi_I}$ of the imitator policy $\pi_I$ is $L_{Q^{\pi_I}}$-LC, then it holds that:
    $$J^{\pi_{E}}-J^{\pi_{I}} \le \frac{L_{Q^{\pi_I}}}{1-\gamma}\E_{s\sim d^{\pi_{E}}}[\wass(\pi_I(\cdot|s), \pi_E(\cdot|s))].$$
\end{thm}
\begin{proof}\label{proof:police}
    Using the \emph{performance difference lemma} \cite{kakade2002approximately}, we have:
    \begin{align*}
    J^{\pi_{E}}-J^{\pi_{I}} =&
    \frac{1}{1-\gamma}\E_{s\sim d^{\pi_{E}}}\left[\E_{a\sim \pi_{E}(\cdot|s)}[A^{\pi_I}(s,a)] \right],
    \end{align*}
    where $A^{\pi_I}(s,a) = Q^{\pi_I}(s,a) - V^{\pi_I}(s)$ is the advantage function. The inner expectation can be written as:
    \begin{align*}
        & \E_{a\sim \pi_{E}(\cdot|s)}[A^{\pi_I}(s,a)]\\
         & = \int_A Q^{\pi_{I}}(s,a)(\pi_{E}(\de a|s)-\pi_{I}(\de a|s)) \\
     & \le \sup_{s \in \Ss} \left\|Q^{\pi_I}(s,\cdot) \right\|_L \wass(\pi_E(\cdot|s), \pi_I(\cdot|s)),
    \end{align*}
    where the inequality follows from the definition of Wasserstein metric. The result is obtained by observing that $\sup_{s \in \Ss} \left\|Q^{\pi_I}(s,\cdot) \right\|_L \le  \left\|Q^{\pi_I} \right\|_L \le L_{Q^{\pi_I}}$.
\end{proof}

A similar bound was previously derived by~\cite[][Theorem 1]{PirottaRB15} and~\cite[][Theorem 2]{AsadiML18}. However, \cite{PirottaRB15} assume that the policy is LC \wrt a policy parametrization. Instead, the result of \cite{AsadiML18} involves the transition model instead of the policy and requires a bound uniform over $\SAs$ on the Wasserstein distance between the true and the estimated models. Let us now revisit Example~\ref{exa:1} in light of Theorem~\ref{police}.

\begin{exa}[continues=exa:1]
Under Equation~\eqref{eq:mse} , we can provide an effective guarantee on the Wasserstein distance:
\begin{align*}
    \E_{s\sim d^{\pi_E}} &  [\wass(\pi_E(\cdot|s), \pi_I(\cdot|s))] = \E_{s\sim d^{\pi_E}} [\left\| \pi_E(s) - \pi_I(s)\right\|_1]  \\
    & \quad \le {\mathbb{E}_{s \sim d^{\pi_E}} [\|\pi_E(s) - \pi_I(s)\|^2_2 ]}^{\frac{1}{2}}\le \varepsilon,
\end{align*}
where in the first inequality, we used Jensen's inequality.
\end{exa}

Comparing Theorem~\ref{police} with Theorem~\ref{thm:old}, we no longer require the uniform bound $R_{\max}$ on the reward function, but we introduce an additional assumption on the regularity of the imitator Q-function $Q^{\pi_I}$. Clearly, we should find suitable assumptions under which $L_{Q^{\pi_I}}$ is finite.
As we anticipated in Section~\ref{sec:prelimMDP}, the only known result that provides such an estimate under the assumption of Lipschitz MDP and Lipschitz policy with Wasserstein metric is \cite{lyp:locality}, where the authors proved that, if $\gamma L_p(1+L_\pi)<1$ is satisfied, $ L_{Q^{\pi}}$ can be chosen as:
\begin{align}\label{eq:boundQ}
  L_{Q^{\pi}} \coloneqq \frac{L_r}{1-\gamma L_p(1+L_\pi)}.  
\end{align}
However, we argue that condition $\gamma L_p(1+L_\pi)<1$ is very demanding and often unrealistic. Indeed, to fulfill it we need at least one of these conditions to be satisfied:
\begin{enumerate}[noitemsep, label=(\roman*), topsep=0pt]
    \item $\gamma \ll 1$: in practice, it is almost always false, since the discount factor is often chosen to be close to $1$;
    \item $L_p < 1$: this is a very unrealistic assumption, since it would make all the states shrink exponentially when the same actions are performed;
    \item $L_{\pi} \approx 0$: the action depends very little on the state so that there is a very limited possibility of controlling the environment (this condition alone is not even sufficient).
\end{enumerate}

\begin{figure}[t]
\centering
\includegraphics[width=\linewidth]{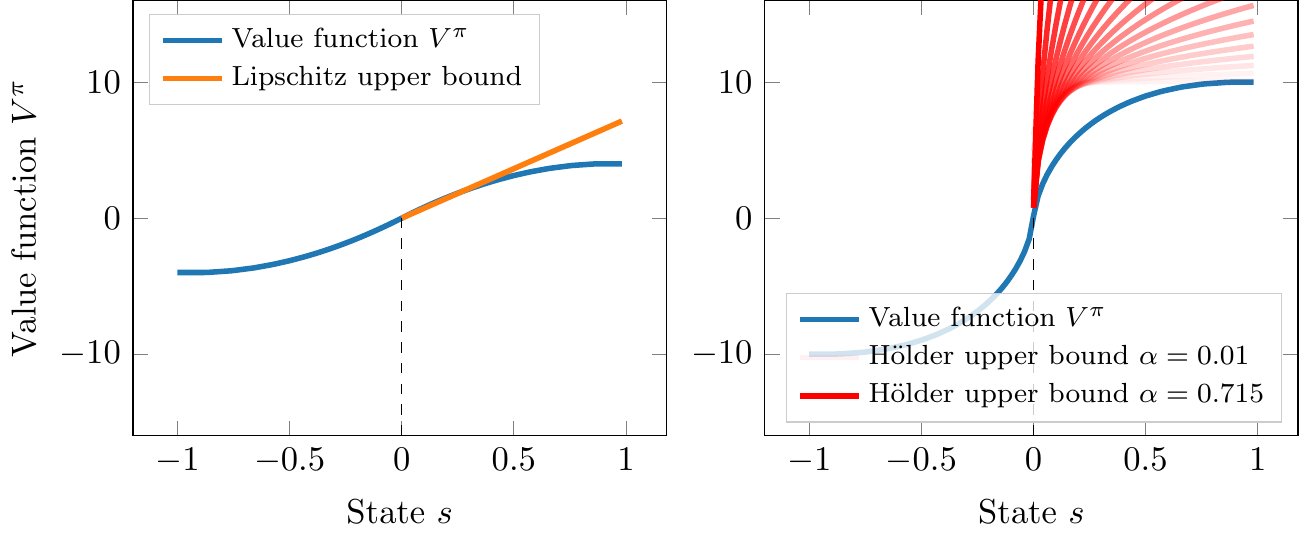}
\caption{State value functions of Example~\ref{ese1}. Left: the bound of \cite{lyp:locality} hold and it is tight. Right: the bound of \cite{lyp:locality} does not hold, but our bound based on H\"older continuity holds, for different values of $\alpha \in (0,1)$.}
\label{exa_fig}
\end{figure}

\subsection{The Tightness of the Value Function Lipschitz Constant}
It is legitimate to question whether the value $L_{Q^{\pi}}$ of Equation~\eqref{eq:boundQ}, widely employed in the literature~\citep[\eg][]{lyp:locality, PirottaRB15, AsadiML18}, is a tight approximation of the Lipschitz semi-norm $\|Q^{\pi}\|_L$. 
In this section, we prove that the result cannot be improved, at least when requiring the Lipschitz continuity of the value function. Example~\ref{ese1} shows that the value function $Q^{\pi}$ can be made non-LC even when the MDP and the policy are LC, while Theorem~\ref{counterex} proves that a bound like the one of Theorem~\ref{police} cannot be obtained for a generic Lipschitz MDP and Lipschitz policies.

\begin{exa}[label=ese1]
    Let $\mathcal{M}$ be an MDP and $\pi$ be a policy defined as follows, given the constants $L_p,L_r > 0$:
    \begin{itemize}[noitemsep, leftmargin=*, topsep=0pt]
        \item $\Ss = [-1,1]$;
        \item $\As = \{0\}$;
        \item The dynamic is deterministic. From every state $s \in \Ss$, performing action $0$, the only possible, the environment moves to the state $s'=\mathrm{clip}(L_ps,-1,1)$.\footnote{$\mathrm{clip}(x,a,b)$ is the \emph{clipping} function, \ie $\max\{\min\{x, b\},a\}$.} This means that $p(\de s'|s,a)=\delta_{\mathrm{clip}(L_ps,-1,1)}(\de s')$;
        \item $r(s,a) = L_rs$;
        \item The initial state distribution is $\mu = \mathrm{Uni}([0,1])$ (not influential for the derivation that follows).
    \end{itemize}
    This MDP is $(L_p,L_r)$-LC and the policy has Lipschitz constant equal to $L_\pi = 0$, since there is one action only. Equation~\eqref{eq:boundQ} ensures that the state value function $V^{\pi}$ (that is equal to the state-action value function $Q^{\pi}$ since there is one action only) is LC with constant:
    $$L_{V^{\pi}} = \frac{L_r}{1-\gamma L_p}.$$
    Since the state space is one dimensional, we can compute the state value function $V^{\pi}$ exactly:
    $$V^{\pi}(s)=L_r\sum_{k=0}^{+\infty} \gamma^k\ \mathrm{clip}{\left(L_p^ks,-1,1\right)}, \qquad \forall s \in \Ss.$$
    As shown in Figure \ref{exa_fig} left, the point of maximal slope $s=0$. Even if we have employed the specific values $L_p=1.15$,  $L_r=1$, and $\gamma=0.75$, it is simple to see that this property is valid in general. Moreover, we have plotted in orange the line which passes through the origin, having slope equal to:
    $$L_{V^{\pi}} = \frac{L_r}{1-\gamma L_p}=\frac{1}{1-0.75\cdot 1.15}\approx 7.27,$$
    which is the tangent line to the state value function in $s=0$, as it also can be found analytically:
    $$\frac{\partial V^{\pi}}{\partial s}(0)=L_r\sum_{k=0}^{+\infty} \gamma^k L_p^k=\frac{L_r}{1-\gamma L_p}.$$
    This means that, in this case, the choice of the Lipschitz constant provided by the theory (Equation~\ref{eq:boundQ}) is actually tight.
    
    What happens if we reach the hard edge of $\gamma L_p (1 + L_\pi) = \gamma L_p> 1$, where Equation~\eqref{eq:boundQ} does not guarantee any property? 
    For instance, by taking $L_p=1.15$, $L_r=1$, and $\gamma=0.9$, we lose any Lipschitz property, finding a derivative which is unbounded, as shown in Figure \ref{exa_fig} right.
\end{exa}

Note that, in this example, we are able to find a non-LC state value function even in the apparently simple case of $\As=\{0\}$, where $L_\pi=0$. Therefore, this example also shows that the dynamics of the system alone is enough to make the state value function irregular. Furthermore, the same example can be adapted to prove that, for a generic Lipschitz MDP and a pair of Lipschitz policies, a bound like the one of Theorem~\ref{police} cannot be obtained in general.

\begin{restatable}[]{thm}{counterex}\label{counterex}
    There exist an $(L_p,L_r)$-LC MDP and an $L_\pi$-LC policy $\pi$ such that for every finite constant $C>0$, (even depending on $L_p$, $L_\pi$, and $L_r$), there exists an $L_{\pi}$-LC policy $\pi'$ such that:
    $$J^{\pi}-J^{\pi'} \ge C\E_{s\sim d^{\pi}}[\wass(\pi(\cdot|s),\pi'(\cdot|s))].$$
\end{restatable}

The proof is reported in Appendix \ref{proof:counterex}.
If we set $\pi = \pi_E$ as the expert policy and $\pi'= \pi_I$ as an imitator policy, Theorem~\ref{counterex} shows that, even if the MDP and the policies are LC, we cannot, in general, upper bound the performance difference $J^{\pi_E}-J^{\pi_I}$ with the expected Wasserstein distance $\E_{s\sim d^{\pi}}[\wass(\pi_E(\cdot|s),\pi_I(\cdot|s))]$. This is in line with the fact that, without additional assumptions, \eg when $\gamma L_p (L_\pi+1)<1$ does not hold, Theorem~\ref{counterex} is vacuous. Therefore, these bounds cannot be improved in the framework of Lipschitz continuity, however, a weaker notion of regularity can be used to generalize the previous theorems.

\section{H\"older Continuity is All We Need}\label{sec:holder}
In this section, we propose an approach for overcoming the limitations of the Lipschitz continuity, discussed in the previous section. In Section~\ref{sec:holdVal}, we show that the state-action value function $Q^{\pi}$ is always H\"older continuous, provided that the MDP and the policy are LC. Then, in Section~\ref{sec:gap2}, we apply these findings to BC, deriving a bound on the performance difference $J^{\pi_E}-J^{\pi_I}$ in terms of the Wasserstein distance that holds for \emph{every} LC MDP and policy.

\subsection{The H\"older Continuity of the Value Function}\label{sec:holdVal}
The first step to improve the result of \cite{lyp:locality} is to observe that, like in Example \ref{ese1}, even when the value function is not Lipschitz continuous, it keeps being continuous. This observation is not, in principle, accounted for by the previous analysis, which provides no result when $\gamma L_p (1 + L_\pi) > 1$. This suggests that employing a notion of regularity that is stronger than continuity but weaker than Lipschitz continuity, as H\"older continuity, might lead to an improvement of the analysis. Indeed, we are able to prove the following generalization.

\begin{restatable}[H\"older-continuity of the Q-function]{thm}{holdercont}\label{holdercont}
    Let $\mathcal{M}$ be an $(L_p,L_r)$-LC MDP, let $\pi$ be an $L_{\pi}$-LC policy, and let
    $$0 < \alpha < \overline{\alpha} \coloneqq \min \left\{ 1, \frac{-\log \gamma}{\log (L_p(1+L_\pi))} \right\}.$$
    If  the state space $\Ss$ and the action space $\As$ admit finite diameter\footnote{The \emph{diameter} of a metric space $(\Xs,d_{\Xs})$ is defined as: $\diam{\Xs} = \sup_{x,x' \in \Xs} d_{\Xs}(x,x')$.} $\diam{\Ss}$ and $\diam{\As}$, respectively, then 
    the state-action value function $Q^\pi$ is $(\alpha,L_{Q^{\pi},\alpha})-$HC with a H\"older constant bounded by:
    $$L_{Q^{\pi},\alpha} \coloneqq \frac{L_r \left(\diam{\Ss} + \diam{\As}\right)^{1-\alpha}}{1-\gamma(L_p(1+L_\pi))^\alpha}.$$
\end{restatable}

The proof is reported in Appendix \ref{app:holder}. The requirement on the finiteness of the diameter of the state and action spaces is a mild condition. Indeed, it is common to assume that these spaces are bounded (or even compact), which implies a finite diameter. Furthermore, as commonly done in practice, one can easily re-scale the states and actions, modifying the diameter and not altering the nature of the problem.\footnote{As an alternative, one could assume that the reward function $r(s,a)$ is $(\alpha,L_r)$-HC, removing the need for the diameters.} 
Furthermore, we can easily obtain the H\"older constant of the state value function $V^{\pi}$.

\begin{restatable}[H\"older-continuity of the V-function]{prop}{QtoV}\label{QtoV}
     Let $\pi$ be an $L_{\pi}$-LC policy. If the state-action value function $Q^{\pi}$ is $(\alpha,L_{Q^{\pi},\alpha})$-HC, then the corresponding state value function $V^{\pi}$ is $(\alpha,L_{V^{\pi},\alpha})$-HC with:
    \begin{align*}
        L_{V^{\pi},\alpha} \coloneqq L_{Q^{\pi},\alpha} (L_{\pi}+1)^\alpha.
    \end{align*}
\end{restatable}

These result represent a generalization of those of \cite{lyp:locality}, which are obtained by setting $\alpha=1$.

Moreover, this Theorem~\ref{holdercont} implies that the value functions of an LC MDP and policy is always continuous, since any HC function is also continuous, regardless of its constants, as it seemed from the previous example.
Coming back to Example~\ref{ese1}, we can perform  further analyses.

\begin{exa}[continues=ese1]
    We can use Theorem \ref{holdercont} to provide an upper bound on the value function even if the Lipschitz continuity does not hold. 
    The critical exponent is given by: 
    $$\overline{\alpha} = -\frac{\log\gamma}{\log(L_p(1+L_\pi))}\approx 0.72.$$
    For every value of $\alpha < \overline{\alpha}$, the state value function $V^{\pi}$ is $(\alpha,L_{V^{\pi},\alpha})-$HC. As we can see in Figure \ref{exa_fig} right, for small $\alpha$, the bound provided by $L_{V^{\pi},\alpha} |s|^\alpha$ is tight for $s\to 1$.
\end{exa}

\subsection{A More General Bound based on Wasserstein Distance}\label{sec:gap2}
Similarly to what we have done in Section~\ref{sec:firstBound}, to a result of regularity, we are able to associate a result about the loss of BC, bounding the difference in performance between two policies with their Wasserstein distance. Indeed, thank to Theorem~\ref{holdercont}, we can prove the following bound.

\begin{restatable}[Optimal Error Rate for BC]{thm}{BCfond}\label{BCfond}
     Let $\pi_E$ be the expert policy and $\pi_I$ be the imitator policy. If that state-action value function $Q^{\pi_I}$ of the imitator policy $\pi_I$ is $(\alpha,L_{Q^{\pi_I},\alpha})$-HC, then it holds that:
     $$ J^{\pi_{E}}-J^{\pi_{I}}\le \frac{L_{Q^{\pi_I},\alpha}}{1-\gamma} \E_{s\sim d^{\pi_{E}}}\left[\wass\left(\pi_E(\cdot|s),\pi_I(\cdot|s)\right)^\alpha \right].$$
     
     Furthermore, if the MDP $\mathcal{M}$ is $(L_p,L_r)$-LC and the imitator policy $\pi_I$ is $L_{\pi_I}$-LC, the bound is tight for what concerns the exponent $\alpha$ that cannot be improved above the critical value $\overline{\alpha}$ of Theorem~\ref{holdercont}.
\end{restatable}

The proof is reported in Appendix ~\ref{app:holder}. This is a remarkable result, since it is valid for every LC MDP and policy. 
As expected, a low value of $\alpha$ leads to a looser bound. Unfortunately, this bound, despite being tight in the exponent, is difficult to manage in practice. Indeed, in order to minimize the right-hand side, $\E_{s\sim d^{\pi_{E}}} [\wass(\pi_E(\cdot|s),\pi_I(\cdot|s))^\alpha  ]$, one should know the value $\alpha$ in advance. However, $\alpha \le \overline{\alpha}$ depends on the Lipschitz constants of the environment and of the policy, which are usually unknown. Therefore, no imitation learning algorithm can be trained to minimize this error explicitly. Fortunately, we can see that, weakening this result, we can obtain a more practical guarantee. Since $0 < \alpha < 1$, we can apply Jensen's inequality to obtain:
\begin{align}\label{eq:Jen}
   \!\!\! J^{\pi_{E}}-J^{\pi_{I}}\le \frac{L_{Q^{\pi_I},\alpha}}{1-\gamma} \E_{s\sim d^{\pi_{E}}}\left[\wass\left(\pi_E(\cdot|s),\pi_I(\cdot|s)\right) \right]^\alpha. \!\!\!
\end{align}

In this formulation, we minimize the expected Wasserstein distance only, and the knowledge of $\alpha$ is not needed, but its value impacts the kind of guarantee we can provide.
\begin{remark}\label{eur1}
 If we perform BC in a $(L_p,L_r)$-LC MDP and with a $L_{\pi_I}$-LC imitator policy $\pi_I$, the best possible performance guarantee (from Equation~\ref{eq:Jen}) is given by:
    $$J^{\pi_{E}}-J^{\pi_{I}} \le \mathcal{O}(\varepsilon^\alpha),$$
    where $\varepsilon$ is the square root of the imitation MSE, \ie $\varepsilon^2 = \E_{s \sim d^{\pi_E}}[\| \pi_E(s) - \pi_I(s)\|^2_2]$ as defined in Example~\ref{exa:1}, and $\alpha <\overline{\alpha} = -\frac{\log\gamma}{\log(L_p(1+L_{\pi_I}))}$, the critical exponent.
\end{remark}

Therefore, a very low value of $\alpha$, corresponding to lack of regularity, can badly influence the possibility of learning a good imitator policy.

\section{Noise Injection}\label{sec:noiseInj}
BC may struggle when the regularity assumptions are lacking. However, in practice, using a \emph{noisy} expert policy may significantly help the learning process~\cite{laskey2017dart}. This empirical benefit is justified by the intuition that noise helps in exploring the neighborhood of the expert trajectories. In this section, we formulate this empirical evidence in a mathematically rigorous way. Indeed, we show how to break the barrier enforced by Theorem~\ref{BCfond}, whose result is obtained by a deterministic expert. Clearly, these advantages come with the price that a noisy expert might experience a loss in expected return compared to the deterministic one.

\begin{figure*}[t]
\centering
\includegraphics[width=0.8\textwidth]{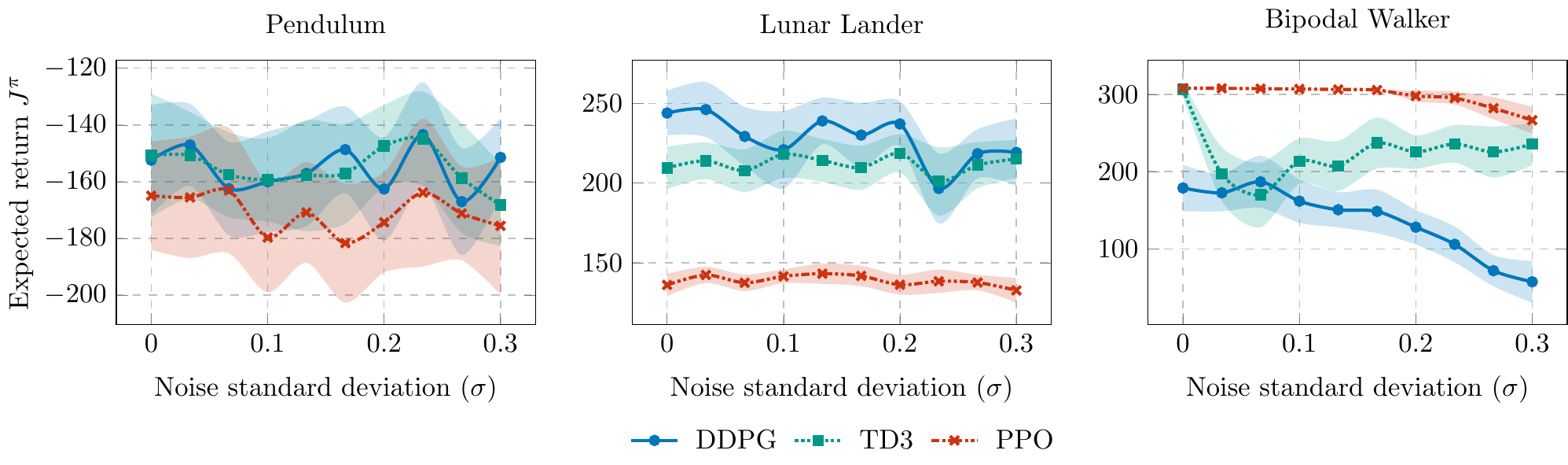}
\caption{The performance of the expert $J^{\pi}$ as a function of the standard deviation of the noise $\sigma$. 
The performance is measured on 40 episodes int environment repeated for 20 different random seeds (nuance represents the $95\%$ non-parametric c.i.).}
\label{noisy_effect}
\end{figure*}

\subsection{Noise Injection: a Mathematical Formulation}
The simplest form of noise injection is realized by adding to the expert's action $a_{E,t}$ a noise component $\eta_t$. In particular, assuming that the action space is real, \ie $\As \subseteq \mathbb{R}^n$, we have:
\begin{align}\label{eq:noisInj}
   \forall t \in \mathbb{N}: \qquad \begin{cases}
   a_{t,E} \sim \pi_E(\cdot|s_t)\\
   \eta_t \stackrel{\text{iid}}{\sim} \mathcal{L} \\
    a_t = a_{t,E} + \eta_t\\
\end{cases},
\end{align}
where $\{\eta_t\}_{t \in \mathbb{N}}$ is a noise sequence whose components are independent between each other and from the sequences of states and actions, and identically distributed by law $\mathcal{L} \in \PM{\As}$. 
If $\mathcal L$ admits a density function, we can express the density function of the played action $a_t$ as the convolution of the expert policy density function $\pi_E$ and the density function $\ell$ of the noise law $\mathcal L$.
Note that the formalization in Equation~\eqref{eq:noisInj} encompasses distributions that do not correspond to the intuitive idea of \emph{noise} (\eg when $\mathcal L$ is a discrete law). To obtain a meaningful result, we enforce the following assumption.
\begin{ass}\label{noiseass}
    The law of the noise $\mathcal L$ admits a density function \wrt a reference measure $\ell : \mathbb{R}^n \rightarrow \mathbb{R}_{\ge 0}$ and is TV-LC (see Definition \ref{tv_lips}) with constant $L_{\ell}$.
\end{ass}
Under this assumption, denoting with $\pi_{E,\ell}$ the policy with noise injection, \ie $a_t \sim \pi_{E,\ell}(\cdot|s_t)$, we have that:
\begin{equation*}\label{noise_def}\pi_{E,\ell}( a|s)=\int_{\R^n} \pi_E( a'|s)\ell(a-a')  \de a', \;\; \forall (s,a) \in \SAs. \end{equation*}
This represents the convolution of the policy density function $\pi_{E}$ and the noise density function $\ell$. In other words, this shows that the action taken by the expert policy $a_{E,t}$ is averaged over the noise probability distribution.

Assumption~\ref{noiseass} covers the most common types of noise, like the Gaussian or the uniform ones. In fact, we can prove that every univariate unimodal distribution satisfies Definition~\ref{tv_lips} (see Proposition \ref{unimodal} in the Appendix \ref{app:math}).
Considering multivariate Gaussian noise, we directly derive the $L_{\ell}$ constant.

\begin{exa}
    Suppose the noise is sampled from a zero-mean Gaussian distribution $\mathcal{N}(0,\Sigma)$ with covariance matrix ${\Sigma}$, the previous integral writes, for all $(s,a) \in \SAs$:
    $$\pi_{E,\ell}(a|s)=\int_{\R^n} \pi_E(a'|s)\underbrace{\frac{e^{-\frac{1}{2} (a-a')^T \Sigma^{-1}(a-a')}}{(2\pi)^{n/2}\mathrm{det}({\Sigma})^{1/2}}}_{\ell(a-a')} \de a',$$
    where we recognise the Gaussian $n$-variate density $\ell$. Assumption \ref{noiseass} is verified since, for $h \in \mathbb{R}^n$:
    \begin{align*}
        \tv{\left(\mathcal{N}(h,\Sigma),\mathcal{N}(0,\Sigma)\right)} & \le \sqrt{\frac{1}{2}\mathrm{KL}{\left(\mathcal{N}(h,\Sigma),\mathcal{N}(0,\Sigma)\right)}}\\
        &= \frac{1}{2} \left\| h \right\|_{\Sigma^{-1}} \le \frac{1}{2 \sqrt{s_{\min}(\Sigma)}} \|h\|_2,
    \end{align*}
    where we used Pinsker's inequality and $s_{\min}(\cdot)$ denoted the minimum singular value of a matrix.
    In particular, if $\Sigma$ is diagonal as $\sigma^2 I$, we have that $L_\ell=1/(2\sigma)$.
\end{exa}

It is worth noting that, in the diagonal covariance case, $L_\ell$ is proportional to $\sigma^{-1}$. This suggests that, the smaller the impact of the noise $\mathcal{L}$, \ie the smaller the standard deviation $\sigma$, the higher the constant $L_\ell$. 
Indeed, as $\sigma$ decreases, the regularization effect of the noise becomes less relevant (in the limit $\sigma\rightarrow 0$, noise injection vanishes).

\subsection{A Bound based on Wasserstein Distance for Noise Injection}
We are now able to prove a performance guarantee for BC with noise injection. The idea is based on a simple yet interesting fact. We can use the noise to smooth a bounded function, as in Proposition \ref{smooth}. Applying this approach to the state-action value function, leads to the following result. 

\begin{restatable}[]{thm}{mainSmooth}\label{main_smooth}
Let $\pi_E$ be the expert policy and $\pi_I$ be the imitator policy. Let us suppose that we have injected a noise of density function $\ell$, satisfying Assumption \ref{noiseass} to obtain a noisy expert $\pi_{E,\ell}$ and a noisy imitator $\pi_{I,\ell}$. If $|Q^{\pi_I}(s,a)| \le Q_{\max}$ for all $(s,a) \in \SAs$, it holds that:
    $$J^{\pi_{E,\ell}}-J^{\pi_{I,\ell}}\le \frac{2 L_\ell Q_{\max}}{1-\gamma}\E_{s\sim d^{\pi_{E,\ell}}}[\wass(\pi_E(\cdot|s),\pi_I(\cdot|s))].$$
\end{restatable}
The proof is reported in Appendix ~\ref{app:noise_inj}.
Some observations are in order. First, note the similarity with Theorem \ref{police}, with the only difference being the substitution of $L_{Q^{\pi}}$ with $2 L_\ell Q_{\max}$. 
Second, we require no smoothness assumption (\eg Lipschitz continuity) on the environment or on the policy.  
Yet, if in the previous result of Theorem \ref{police} the constant $L_{Q^{\pi}}$ could easily become infinite, now, the constant $2L_\ell Q_{\max}$ can be easily bounded by $\frac{2L_{\ell}R_{\max}}{(1-\gamma)^2}$, since $Q_{\max} \le \frac{R_{\max}}{1-\gamma}$. From an intuitive perspective, the need for smoothness in the environment is replaced with an assumption on the density function of the noise. Lastly, we can see that, in the right-hand side of the formula, the error is measured by the Wasserstein distance of the non-noisy policies. This is advisable, since it implies that the intrinsic error due to the noise does not affect the bound besides on the $\gamma$-discounted visitation distribution. We show in Appendix that this quantity is always smaller than its counterpart involving the noisy policies.

\begin{remark}\label{eur2}
    If we perform BC injecting a noise $\eta_t$ of density function $\ell$ and satisfying Assumption~\ref{noiseass}, we have the following performance guarantee:
    $$J^{\pi_{E,\ell}}-J^{\pi_{I,\ell}} \le \mathcal{O}(\varepsilon),$$
    where $\varepsilon$ is the MSE of the imitation policy as in Remark~\ref{eur1}.
\end{remark}

In comparison with Remark \ref{eur1} for standard BC, we can appreciate that, here, the exponent $\alpha$ disappeared. Indeed, we have a performance bound that decreases linearly in the MSE. In many cases, when the environment is not intrinsically very smooth, or the expert policy is irregular, the $\alpha$ parameter can be very small, slowing down the convergence significantly. Instead, a liner decay is a relevant improvement of the v speed. Furthermore, as already noted, no assumption of regularity is required in Theorem~\ref{main_smooth}, so that the last result has a much wider range of applications.

\section{Practical Considerations}
In the previous sections, we have seen that the use of noise injection allows having a much better performance guarantee than standard BC (see Remarks \ref{eur1} and \ref{eur2}). Still, in practice, what matters is to have an imitator policy that is good itself rather than an imitator that is simply good in mimicking a given policy. Therefore, if with the noise injection we negatively affect the performance of the expert, \ie if $J^{\pi_{E,\ell}} \ll J^{\pi_E}$, the results given about noise injection could become useless.
On the contrary, we argue that adding noise to the expert's action to a certain extent, does not particularly affect performance. In Figure \ref{noisy_effect}, we show the results of testing this statement on some of the most common continuous-actions environments of the \texttt{OpenAI gym}~\cite{brockman2016gym} library.
In this simulation,\footnote{Details can be found in Appendix.} we first train an expert policy with the well-known DDPG~\cite{ddpg}, TD3~\cite{fujimoto2018addressing} and PPO ~\cite{schulman2017proximal} in the following \texttt{OpenAI gym} environments:
\begin{itemize}[noitemsep, topsep=0pt]
    \item \texttt{Pendulum-v0}: this environment has a continuous action space $[-2,2]$. The objective is to apply torque on a pendulum to swing it into an upright position. The whole system is very regular, as it is governed by simple differential equations, and is also deterministic, except for the initial position of the pendulum, which is random.
    \item \texttt{LunarLanderContinuous-v2}: this environment has a continuous action space $[-1,1]^2$. Here, we have to make a rocket land safely in a landing pad. The dynamics is quite complex, and stochasticity is present to simulate the effect of the wind.
    \item \texttt{BipedalWalker-v3}: this environment has a continuous action space $[-1,1]^4$. Here we have to make a bipedal robot walk. The dynamics is even more complex, but the whole system is deterministic.
\end{itemize}
Then, we evaluated the performance of these experts with noise injection with Gaussian noise with different standard deviations. 
As we can see in figure \ref{noisy_effect}, even when the noise increases until it is close to the radius of the action space, at least in seven cases out of nine, the performance does not suffer significant drops. Intuitively, this can be explained by the fact that we applied an i.i.d. zero-mean noise sequence that is independent of the state and the action. Thus, its effect does not accumulate over the horizon.

\section{Conclusions}
In this paper, we have addressed BC in the case of continuous-action environments from a theoretical perspective. We have shown that the existing theoretical guarantees on BC are not suitable when dealing with continuous actions. Thus, we have derived a first bound for the performance guarantees, under the assumption that the imitator value function is Lipschitz continuous. Since this latter assumption is demanding (\ie it is not guaranteed even when the underlying MDP and policy are LC), we have relaxed it by studying the continuity properties of the value function. As a result of independent interest, we have proved that the value function is always H\"older continuous, under the milder assumption that the underlying MDP and policy are LC. Then, we have applied these findings to obtain a general bound for the performance gap of BC, which we have proved to be tight. Finally, we have formalized noise injection and we have shown the advantages of this practice when applied to BC.

\bibliography{aaai23}

%
%

\newpage
\appendix
\onecolumn

\section{General Math}
\label{app:math}
\smooth*

\begin{proof}
For every $x,y \in \Xs$, we have:
    \begin{align*}
        |(f*\ell)(x)-(f*\ell)(y)| &= \bigg |\int_{\mathbb R^n}f(z)g(x-z) \de z -\int_{\mathbb
        R^n}f(z)g(y-z)\ \de z\bigg|\\
        &= \bigg |\int_{\mathbb R^n}f(z)(\ell(x-z)-\ell(y-z))\ \de z\bigg|,
    \end{align*}
    since $\|f\|_\infty \le M$, by assumption, we have that the result is bounded by:
    $$M \int_{\mathbb{R}^n} \left|\ell(x-z)-\ell(y-z)\right| \de z =2 M \tv(\mathcal{L}(x-\cdot)-\mathcal{L}(y-\cdot))\le 2LM\|x-y\|_2,$$
    where the last passage follows from the definition of $L$-TV-LC.
\end{proof}

\begin{prop}\label{unimodal}
    Let $f$ be any univariate density function. Then, if $f$ is unimodal (i.e. there is $x_0$ such that $f(x)$ is nondecreasing in $(-\infty, x_0]$ and nonincreasing in $[x_0, +\infty)$), $f$ is TV-LC with $L=2\sup_{x\in \mathbb R} f(x)$.
\end{prop}
\begin{proof}

    We first prove the result for functions $f$ that are Lipschitz with constant $L_f$:
    
    \begin{align*}
        \tv(f(\cdot + h) , f(\cdot))& = \int_{\mathbb R} |f(x+h)-f(x)|\ \de x\\
        & = \int_{-\infty}^{x_0-h} |f(x+h)-f(x)|\ \de x+\int_{x_0-h}^{x_0} |f(x)-f(x+h)|\ \de x+\int_{x_0}^\infty |f(x)-f(x+h)|\ \de x.
    \end{align*}
    At this point, we know that being $f$ continuous there is one point in $x\in [{x_0-h},{x_0})$ such that $f(x)=f(x+h)$. Therefore, by Lipschitzness, in the whole interval $[{x_0-h},{x_0})$ we have $|f(x)-f(x+h)|\le 2L_fh$.
    Substituting in the previous formula, we have:
    $$\le \int_{-\infty}^{x_0-h} |f(x+h)-f(x)|\ \de x+2L_fh^2+\int_{x_0}^\infty |f(x)-f(x+h)|\ \de x.$$
    Now, note that:
    \begin{itemize}
        \item in the interval $(-\infty, x_0-h]$ the function $f(x+h)-f(x)$ is nonnegative;
        \item in the interval $[x_0,+\infty)$ the function $f(x+h)-f(x)$ is nonpositive.
    \end{itemize}
    Therefore, the previous integral writes:
    \begin{align*}&\int_{-\infty}^{x_0-h} f(x+h)-f(x)\ \de x+\varepsilon+\int_{x_0}^\infty f(x)-f(x+h)\ \de x\\
    &\le 2L_fh^2 + \int_{x_0}^{x_0-h} f(x)\ \de x + 
    \int_{x_0+h}^{x_0} f(x)\ \de x\le 2\sup f h+ 2L_fh^2.
    \end{align*}
    This being valid for every $h$, we can also write, by triangular inequality, for every $K\in \mathbb N$:
    \begin{align*}\tv(f(\cdot + h) , f(\cdot))
    &\le \sum_{k=0}^{K-1}
    \tv\bigg (f\Big(\cdot + h\frac{k}{K}\Big ) , f\Big (\cdot + h\frac{k+1}{K}\Big)\bigg)\\
    & \le \sum_{k=0}^{K-1}
    2\sup f\frac{h}{K}+ 2L_f\frac{h^2}{K^2}
    = 2h\sup f+ 2L_f\frac{h^2}{K}.
    \end{align*}
    This entails, taking $K\to \infty$, that
    $$\tv(f(\cdot + h) , f(\cdot))\le 2h\sup f.$$
    Let us come to the general case. Let $N\in \mathbb N$. As Lipschitz functions are dense in $L^1(\mathbb R)$, taking $\varepsilon = \frac{h}{N}$, we know that there is a Lipschitz function $f_\varepsilon$ satisfying the unimodality and such that $\sup f = \sup f_\varepsilon$ such that
    $$\tv(f,f_\varepsilon)<\varepsilon.$$
    Therefore,
    $$\tv(f(\cdot + h) , f(\cdot))\le 2\varepsilon + \tv(f_\varepsilon(\cdot + h) , f_\varepsilon(\cdot))\le 2\varepsilon + 2h\sup f \le 2h(\sup f + 1/N).$$
    Taking $N\to \infty$, we get the result.
\end{proof}

\section{Negative Result}

\counterex*

\begin{proof}\label{proof:counterex}
    In order to show the result we will use a similar MDP as the one defined in Example~\ref{ese1} of the main paper. Then, we will take $\pi=\pi^*$, the optimal policy, which is constant equal to zero (so $L_\pi=0$), and, as $\pi'$ a sequence of policies $\pi_n$ such that:
    $$\lim_{n\rightarrow+\infty} \frac{J^*-J^{\pi_n}}{\E_{s\sim d^{\pi_{n}}}[\wass(\pi^*(\cdot|s),\pi_n(\cdot|s))]}=+\infty.$$
    In this way, for any finite constant $C>0$, there will be an $n>0$ such that, choosing $\pi'=\pi_n$, we have the result. Let $\MDP$ an MDP defined as follows:
    \begin{itemize}
        \item $\Ss = [0,1]$;
        \item $\As = [0,1]$;
        \item The dynamic is deterministic: from every state $s \in \Ss$, performing action $a \in \As$, we go to the state $s'=\mathrm{clip}(L_p(s+a),0,1)$. This means that $p(\de s'|s,a)=\delta_{\mathrm{clip}(L_p(s+a),0,1)}(\de s')$;
        \item $r(s,a) = -L_r s$ with $L_r > 0$;
        \item The initial state is $0$ (\ie $\mu = \delta_{0}$);
        \item The discount factor is $\gamma$.
    \end{itemize}
    
    Since the reward is always negative, the optimal policy $\pi^*$ is identically zero, since is the only one allowing to remain in the origin. Consider instead the sequence of policies for $n \in \mathbb{N}$, defined as:
    $\pi_n(\cdot|s) = \delta_{1/n}(\cdot)$. Both $\pi^*$ and $\pi_n$ are Lipschitz with $L_\pi=0$, and we have $\wass(\pi^*(\cdot|s),\pi_n(\cdot|s))=1/n$ in every state $s\in \Ss$. Nonetheless, we can see that:
    $$V^*(0)=0\qquad V^{\pi_n}(0)\le -L_r\sum_{k=0}^{+\infty} \gamma^k\ \mathrm{clip}\Big (\frac{L_p^k}{n},0,1\Big ).$$
    When $n\to \infty$, we have:
    $$\lim_{n \rightarrow + \infty} \frac{J^*-J^{\pi_n}}{\E_{s\sim d^{\pi_{n}}}[\wass(\pi^*(\cdot|s),\pi_n(\cdot|s))]}= -\lim_{n \rightarrow + \infty} \frac{V^{\pi_n}(0)}{1/n} \ge \lim_{n \rightarrow + \infty}  n L_r\sum_{k=0}^{+\infty} \gamma^k\ \mathrm{clip}\Big (\frac{L_p^k}{n},0,1\Big ).$$
    Here, we can substitute $\delta \coloneqq 1/n$, finding
    $$\lim_{n \rightarrow + \infty} \frac{J^*-J^{\pi_n}}{\E_{s\sim d^{\pi_{n}}}[\wass(\pi^*(\cdot|s),\pi_n(\cdot|s))]} \ge \lim_{\delta\to 0^+}  \frac{L_r\sum_{k=0}^\infty \gamma^k\ \mathrm{clip}\Big (\delta L_p^k,0,1\Big )}{\delta}=L_r \sum_{k=0}^{+\infty} \gamma^k L_p^k=\frac{L_r}{1-\gamma L_p},$$
    which gives $+\infty$ for $L_p\ge 1/\gamma$.
\end{proof}

\section{Proofs with H\"older continuity}\label{app:holder}

Before proving the main result, we need two technical lemmas.

\begin{lem}\label{arf}
    Let $(\Xs,d_{\Xs}),(\Ys,d_{\Ys})$ be two metric spaces, and $\{f_n\}_{n \in \mathbb{N}}$ be a sequence of functions $\Xs \to \Ys$ that are $(\alpha,L)-$HC such that $\lim_{n \rightarrow +\infty} f_n = f$. Then, $f$ is $(\alpha,L)-$HC.
\end{lem}
\begin{proof}
    Let $x,x'\in \Xs$. Fix $\varepsilon>0$. Let:
    $$n_0 \coloneqq  \inf\{n: \|f_n-f\|_\infty \coloneqq \sup_{x'' \in \Xs} d_{\Ys}(f_n(x''), f(x'')) \le \varepsilon Ld_{\Xs}(x,x')^\alpha\}.$$
    Note that $n_0$ exists by definition. Then, we have:
    \begin{align*}
    d_{\Ys}(f(x),f(x')) &\le d_{\Ys}(f(x),f_{n_0}(x))+d_{\Ys}(f_{n_0}(x),f_{n_0}(x'))+d_{\Ys}(f(x'),f_{n_0}(x'))\\
    & \le d_{\Ys}(f_{n_0}(x),f_{n_0}(x')) + 2\varepsilon L d_{\Xs}(x,x')^\alpha\\
    & \le (2\varepsilon + 1) Ld_{\Xs}(x,x')^\alpha.
    \end{align*}
    Since this is valid for every $\varepsilon > 0$, we have also $d_{\Ys}(f(x),f(x'))\le Ld_{\Xs}(x,x')^\alpha$, which is the thesis.
\end{proof}

\begin{prop}\label{cielo}
    Let $\mu, \nu \in \PM{\Xs}$ be two probability measures on a metric space $(\Xs,d_{\Xs})$. Then, for any $(\alpha,L)$-HC function $f$, it holds that:
    $$\int_{\Xs} (\mu(\de x)-\nu(\de x)) f(x)  \le L \ \wass(\mu, \nu)^\alpha.$$
\end{prop}
For this proof, we thank an interesting Stack Exchange discussion \cite{YuvalPeres}.
\begin{proof}
    Let $K(\mu,\nu)$ denote the space of couplings of $\mu$ and $\nu$, i.e. Borel probability measures on $\Xs \times \Xs$ that project to $\mu$ in the first coordinate and to $\nu$ in the second coordinate. Recall that: 
    $$\wass(\mu,\nu)=\inf_{\lambda \in K(\mu,\nu)} \, \,  \left\{ \int_{\Xs \times \Xs} d_X(x,y) \,\lambda(\de x,\de y)\right\} \,.$$
    Suppose that $\forall x,y \in \Xs$ we have that $ |f(x)-f(y)|\le L d_{\Xs}(x,y)^\alpha$, where $0<\alpha<1$. Then for any $\lambda \in K(\mu,\nu)$, we have: 
    \begin{align*} \left|\int_{\Xs} ( \mu(\de x)- \nu(\de x)) f(x) \right|& \le \int_{\Xs \times \Xs} |f(x)-f(y)| \, \lambda(\de x,\de y)  \\
    & \le L \int _{\Xs \times \Xs} d_{\Xs}(x,y)^\alpha \, \lambda(\de x,\de y) \\ 
    &\le L \left(\int _{\Xs \times \Xs} d_X(x,y) \, \lambda(\de x,\de y) \right)^\alpha,
    \end{align*}
    where the last inequality is an application of H\"older's inequality for the functions $(x,y) \mapsto d_X(x,y)^\alpha$ and the constant $1$, with exponents $p=1/\alpha$ and $q=1/(1-\alpha)$. Alternatively, the last inequality in can be obtained from an application of Jensen's inequality for the convex function $t \mapsto t^{1/\alpha}$ on $[0,\infty)$. Taking infimum over $\lambda \in K(\mu,\nu)$ gives:
    $$\left|\int_{\Xs} ( \mu(\de x)- \nu(\de x)) f(x) \right| \le LW(\mu,\nu)^\alpha \,,$$
    as required.
\end{proof}

\QtoV*

\begin{proof}
    By definition, for every $s \in \Ss$:
    $$V^\pi(s)=\int_{\As} Q^{\pi}(s,a)\pi(\de a|s).$$
    Therefore, by introducing suitable Dirac deltas, we have for $s,s' \in \Ss$:
    \begin{align*}
    \left|V^\pi(s)-V^\pi(s')\right|&=\left|\int_{\As} Q^{\pi}(s,a)\pi(\de a|s)
    -\int_{\As} Q^{\pi}(s',a)\pi(\de a|s') \right|\\
    &=\bigg |\int_{\As}\int_\Ss Q^{\pi}(z,a)\delta_s(\de z)\pi(\de a|s)
    -\int_{\As} \int_{\Ss} Q^{\pi}(z,a)\delta_{s'}(\de z)\pi(\de a|s')\bigg |\\
    & = \bigg |\int_{\As}\int_\Ss \left(\delta_s(\de z)\pi(\de a|s)
    - \delta_{s'}(\de z)\pi(\de a|s') \right) Q^{\pi}(z,a) \bigg |.
    \end{align*}
    Now, thanks to Proposition \ref{cielo}, we have for $s,s' \in \Ss$:
    \begin{align*}
        \left|V^\pi(s)-V^\pi(s')\right|\le L_{Q^{\pi},\alpha} \wass\left(\delta_s(\cdot)\pi(\cdot|s), \delta_{s'}(\cdot)\pi(\cdot|s')\right)^\alpha \le L_\alpha (1+L_\pi)^\alpha d_{\Ss} \left(s,s'\right)^\alpha,
    \end{align*}
    where the last passage follows from the following manipulation of the Wasserstein distance:
    \begin{align*}
        \wass(\delta_s(\cdot)\pi(\cdot|s), \delta_{s'}(\cdot)\pi(\cdot|s')) \le \wass(\delta_s(\cdot)\pi(\cdot|s), \delta_{s}(\cdot)\pi(\cdot|s')) + \wass(\delta_s(\cdot)\pi(\cdot|s'), \delta_{s'}(\cdot)\pi(\cdot|s')).
    \end{align*}
    Concerning the first term, we have:
    \begin{align*}
        \wass(\delta_s(\cdot)\pi(\cdot|s), \delta_{s}(\cdot)\pi(\cdot|s'))  & = \sup_{\|f\|_L \le 1} \left| \int_{\Ss \times \As} \left( \delta_s(\de z)\pi(\de a|s) - \delta_{s}(\de z )\pi(\de a|s') \right) f(z,a) \right| \\
        & = \sup_{\|f\|_L \le 1} \left| \int_{\As} \left(\pi(\de a|s) - \pi(\de a|s') \right) f(s,a) \right| \\
        & \le \sup_{\|f\|_L \le 1}  \|f(s,\cdot)\|_{L} \cdot \wass\left(\pi(\cdot|s), \pi(\cdot|s') \right) \\
        & \le \wass\left(\pi(\cdot|s), \pi(\cdot|s') \right) \le L_{\pi} d_{\Ss}\left(s,s'\right),
    \end{align*}
    having observed that $ \|f(s,\cdot)\|_{L} \le \|f\|_L \le 1$. Concerning the second term, we have:
    \begin{align*}
        \wass(\delta_s(\cdot)\pi(\cdot|s'), \delta_{s'}(\cdot)\pi(\cdot|s')) & = \sup_{\|f\|_L \le 1} \left| \int_{\Ss \times \As} \left( \delta_s(\de z)\pi(\de a|s') - \delta_{s'}(\de z )\pi(\de a|s') \right) f(z,a) \right| \\
        & = \sup_{\|f\|_L \le 1} \left| \int_{\Ss} \left( \delta_s(\de z) - \delta_{s'}(\de z )\right) \int_{\As} \pi(\de a|s')  f(z,a) \right| \\
        & \le \sup_{\|f\|_L \le 1} \left\|\int_{\As} \pi(\de a|s')  f(\cdot,a)  \right\|_L  \wass\left(\delta_{s}(\cdot),\delta_{s'}(\cdot) \right) \\
        & \le d_{\Ss}\left(s,s'\right),
    \end{align*}
    where we observed that the Lipschitz semi-norm is bounded by $1$ since for $z,z' \in \Ss$:
    \begin{align}
        \left|\int_{\As} \pi(\de a|s')  f(z,a) -\int_{\As} \pi(\de a|s')  f(z',a)\right| \le \int_{\As} \pi(\de a|s')  \left|f(z,a)-f(z',a) \right| \le d_{\Ss}(z,z').
    \end{align}
\end{proof}

\holdercont*

\begin{proof}
    Consider the following sequences of functions for $n \in \mathbb{N}$:
    $$
    \begin{cases}
        Q_0(s,a) = 0\\
        Q_{n+1}(s,a) = r(s,a) +
        \gamma \int_{\Ss} V_n(s')p(\de s'|s,a)
    \end{cases}, \qquad
    \begin{cases}
        V_0(s) = 0\\
        V_{n+1}(s) = \int_{\As} Q_{n+1}(s,a)\pi(\de a|s)
    \end{cases}.
    $$
    We want to prove, by induction, that $Q_n$ is always $(\alpha, L_{Q^{\pi},\alpha})-$HC.
    The base case $n=0$ is trivial, since a constant function is H\"older continuous for every couple of parameters. Now, let us suppose that $Q_n$ is $(\alpha, L_{Q^{\pi},\alpha})-$HC.
    Then, for $s_1,s_2 \in \Ss$ and $a_1,a_2 \in \As$, we have:
    \begin{align*}
        \left|Q_{n+1}(s_1,a_1) - Q_{n+1}(s_2,a_2) \right| &=
        \bigg |
        r(s_1,a_1)+\gamma \int_{\Ss} V_n(s')p( \de s'|s_1,a_1)
        - r(s_2,a_2) - \gamma \int_{\Ss} V_n(s')p(\de s'|s_2,a_2)
        \bigg |\\
        & =\bigg |r(s_1,a_1)-r(s_2,a_2)+
        \gamma \int_{\Ss} V_n(s') \Big (p(\de s'|s_1,a_1)-p(\de s'|s_2,a_2)\Big )\bigg |\\
        &\le 
         |r(s_1,a_1)-r(s_2,a_2)|+
        \gamma \bigg |\int_{\Ss} V_n(s') \Big (p(\de s'|s_1,a_1)-p(\de s'|s_2,a_2)\Big )\bigg |.
    \end{align*}
    Now, we consider one term at a time. Concerning the first term, recalling that $0<\alpha \le 1$ and that the reward function is $L_{r}$-LC, we have:
    \begin{align*}
        |r(s_1,a_1)-r(s_2,a_2)| & \le L_r \left(d_{\Ss}(s_1,s_2) + d_{\As}(a_1,a_2) \right) \\
        & = L_r \left(d_{\Ss}(s_1,s_2) + d_{\As}(a_1,a_2) \right)^{\alpha} \left(d_{\Ss}(s_1,s_2) + d_{\As}(a_1,a_2) \right)^{1-\alpha} \\
        & \le L_r \left(\diam{\Ss} + \diam{\As} \right)^{1-\alpha} \left(d_{\Ss}(s_1,s_2) + d_{\As}(a_1,a_2) \right)^{\alpha},
    \end{align*}
    having observed that the distance between any pair of points is smaller than the diameter of the corresponding set. Concerning the second term, we make use of Proposition \ref{QtoV}, this entails that $V_n$ is $(\alpha,L_{Q^{\pi},\alpha}(L_\pi+1)^\alpha)-$HC, Proposition~\ref{cielo}, and exploit that the transition model is $L_p$-LC:
    \begin{align}
         \bigg |\int_{\Ss} V_n(s') \Big (p(\de s'|s_1,a_1)-p(\de s'|s_2,a_2)\Big )\bigg | & \le L_{V^{\pi},\alpha}  \wass \left(p(\cdot|s_1,a_1),p(\cdot|s_2,a_2)\right)^\alpha \\
         & \le L_{Q^{\pi}_n,\alpha}\left(L_p(L_\pi+1)\right)^\alpha \left(d_{\Ss}(s_1,s_2) + d_{\As}(a_1,a_2) \right)^{\alpha}.
    \end{align}
    Thus, we have the recurrence involving the H\"older constants:
    \begin{align*}
        & L_{Q^{\pi}_{0},\alpha} = 0\\
        & L_{Q^{\pi}_{n+1},\alpha} =  L_r \left(\diam{\Ss} + \diam{\As} \right)^{1-\alpha} + \gamma  L_{Q^{\pi}_n,\alpha}\left(L_p(L_\pi+1)\right)^\alpha.
    \end{align*}
    The sequence is convergent for $\left(L_p(L_\pi+1)\right)^\alpha < 1$, which leads to the condition:
    \begin{align*}
        \alpha < \frac{-\log \gamma}{\log (L_p(L_\pi+1))}.
    \end{align*}
    Under such a condition, the limit $L_{Q^{\pi},\alpha}$ can be easily found as:
    \begin{align*}
        L_{Q^{\pi},\alpha} =  L_r \left(\diam{\Ss} + \diam{\As} \right)^{1-\alpha} + \gamma  L_{Q^{\pi},\alpha}\left(L_p(L_\pi+1)\right)^\alpha \implies L_{Q^{\pi},\alpha} = \frac{L_r \left(\diam{\Ss} + \diam{\As} \right)^{1-\alpha}}{1-\gamma\left(L_p(L_\pi+1)\right)^\alpha}.
    \end{align*}
\end{proof}

\BCfond*
\begin{proof}
The first part of the theorem is obtained from the \emph{performance difference lemma} \cite{kakade2002approximately}, we have:
    \begin{align*}
    J^{\pi_{E}}-J^{\pi_{I}} =&
    \frac{1}{1-\gamma}\E_{s\sim d^{\pi_{E}}}\left[\E_{a\sim \pi_{E}(\cdot|s)}[A^{\pi_I}(s,a)] \right],
    \end{align*}
    where $A^{\pi_I}(s,a) = Q^{\pi_I}(s,a) - V^{\pi_I}(s)$ is the advantage function. The inner expectation can be written as:
    \begin{align*}
        \E_{a\sim \pi_{E}(\cdot|s)}[A^{\pi_I}(s,a)] & = \int_A Q^{\pi_{I}}(s,a)(\pi_{E}(\de a|s)-\pi_{I}(\de a|s)) \\
    & \le \sup_{s \in \Ss} L_{Q^{\pi_I}(s,\cdot),\alpha}  \cdot \wass(\pi_E(\cdot|s), \pi_I(\cdot|s))^\alpha,
    \end{align*}
    where the inequality follows from Proposition~\ref{cielo}. The result is obtained by observing that $ \sup_{s \in \Ss} L_{Q^{\pi_I}(s,\cdot),\alpha} \le L_{Q^{\pi_I},\alpha}$.

    For what concerns the second part of the theorem, we build a counterexample. Moreover, we will not focus on the constant, since it is very difficult to estimate it correctly, and we will limit to verify that the bound on the exponent $\alpha$ is tight.
    
    As in the previous results, in order to show that the lower bound is tight, we rely on the process defined in the examples of the main paper.  Let $\MDP$ an MDP defined as follows:
    \begin{itemize}
        \item $\Ss = [0,1]$;
        \item $\As = [0,1]$;
        \item The dynamic is deterministic: from every state $s \in \Ss$, performing action $a \in \As$, we go to the state $s'=\mathrm{clip}(L_p(s+a),0,1)$. This means that $p(\de s'|s,a)=\delta_{\mathrm{clip}(L_p(s+a),0,1)}(\de s')$;
        \item $r(s,a) = -L_r s$ with $L_r > 0$;
        \item The initial state is $0$ (\ie $\mu = \delta_{0}$);
        \item The discount factor is $\gamma$.
    \end{itemize}
    
    As before, the optimal policy $\pi^*$ is identically zero, since is the only one allowing to stay in the origin. Consider instead the sequence of policies $\pi_n(\cdot|s) = \delta_{1/n}(\cdot)$.
    Both $\pi^*$ and $\pi_n$ are Lipschitz with $L_\pi=0$, and we have $\wass(\pi^*(\cdot|s),\pi_n(\cdot|s))=1/n$ in every state $s\in [0,1]$. Nonetheless, we can see that:
    $$V^*(0)=0\qquad V^{\pi_n}(0)\le -L_r\sum_{k=0}^{+\infty} \gamma^k\ \mathrm{clip}\left(\frac{L_p^k}{n},0,1\right).$$
    At this point, we want to evaluate the difference
    $J^*-J^{\pi_n}$ as $n$ increases:
    $$J^*-J^{\pi_n}
    \le L_r\sum_{k=0}^{+\infty} \gamma^k\ \mathrm{clip}\Big (\frac{L_p^k}{n},0,1\Big )\overset{\delta=1/n}{=}L_r\sum_{k=0}^{+\infty} \gamma^k\ \mathrm{clip}\Big (\delta L_p^k,0,1\Big ).$$
    At this point, we can see that:
    $$
    \mathrm{clip}\Big (\delta L_p^k,0,1\Big )
    =
    \begin{cases}
        \delta L_p^k & \text{if } k\le -\frac{\log \delta}{\log L_p}\\
        1 & \text{if } k>-\frac{\log \delta}{\log L_p}
    \end{cases}.
    $$
    Therefore, the previous sum can be rewritten as:
    $$\sum_{k=0}^{+\infty} \gamma^k\ \mathrm{clip}\Big (\delta L_p^k,0,1\Big ) = \delta \sum_{k=0}^{\big\lfloor -\frac{\log \delta}{\log L_p} \big\rfloor}\gamma^kL_p^k + \sum_{k=\big\lfloor -\frac{\log \delta}{\log L_p} \big\rfloor+1}^\infty \gamma^k.$$
    Here, using the formula for geometric sums, we have:
    \begin{align*}\sum_{k=0}^{+\infty} \gamma^k\ \mathrm{clip}\Big (\delta L_p^k,0,1\Big ) & \le
    \delta 
    \frac{(\gamma L_p)^{-\frac{\log \delta}{\log L_p}}-1}{\gamma L_p-1}
    + \frac{\gamma^{-\frac{\log \delta}{\log L_p}}}{1-\gamma}\\
    &=\delta 
    \frac{\delta^{-1}\delta^\frac{-\log \gamma}{\log L_p}-1}{\gamma L_p-1}
    + \frac{\delta^\frac{\log \gamma}{\log L_p}}{1-\gamma}=
    \frac{\delta^\frac{-\log \gamma}{\log L_p}-1}{-\gamma L_p-1}
    + \frac{\delta^\frac{-\log \gamma}{\log L_p}}{1-\gamma}.
    \end{align*}
    Since $\gamma^{-\frac{\log \delta}{\log L_p}}=e^{-\log(\gamma)\frac{\log \delta}{\log L_p}}=\delta^\frac{-\log \gamma}{\log L_p}$. This means that, as we have $\wass(\pi^*(\cdot|s),\pi_n(\cdot|s))=1/n = \delta$     in every state, $\delta$ also corresponds to the Wasserstein error over the trajectory. Therefore, Theorem \ref{BCfond} provides
    $$J^*-J^{\pi_n}\le C\delta^\alpha \qquad \text{with} \qquad
    \alpha\le \overline{\alpha} \coloneqq \frac{-\log \gamma}{\log (L_p(1+L_\pi))}=\frac{-\log \gamma}{\log (L_p)},$$
    while we have just found
    $$
    J^*-J^{\pi_n} \ge c\ \delta^\frac{-\log \gamma}{\log L_p},
    $$
    which makes the value of $\overline{\alpha}$ tight.
\end{proof}

\section{Results with Noise Injection}

\label{app:noise_inj}

\mainSmooth*

\begin{proof}
    By the performance difference lemma, we have
    $$J^{\pi_{E,\ell}}-J^{\pi_{I,\ell}}=
    \frac{1}{1-\gamma}\E_{s\sim d^{\pi_{E,\ell}}}\Big [\E_{a\sim \pi_{E,\ell}(\cdot|s)}[Q^{\pi_{I,\ell}}(s,a)]-\E_{a\sim \pi_{I,\ell}(\cdot|s)}[Q^{\pi_{I,\ell}}(s,a)] \Big ].$$
    The inner part can be written as:
    $$\int_A Q^{\pi_{I,\ell}}(s,a)(\pi_{E,\ell}(a|s)-\pi_{I,\ell}(a|s))\ da.$$
    Here, by definition we have $\pi_{E,\ell}(a|s)=\int_A \pi_E(a'|s)\ell(a-a')\ da'$, and the same for $\pi_{I,\ell}$. Therefore, we find
    \begin{align*}
    \int_A & Q^{\pi_{I,\ell}}(s,a)(s,a)\bigg(\int_A \ell(a-a')(\pi_E(a'|s)-\pi_I(a'|s))\ da'\bigg)\ da \\
    &\qquad\quad =\int_A (\pi_E(a'|s)-\pi_I(a'|s))\bigg(\int_A Q^{\pi_{I,\ell}}(s,a) \ell(a-a')\ da\bigg)\ da'.  
    \end{align*}
    Now, note that the inner term $\int_A Q^{\pi_{I,\ell}}(s,a) \ell(a-a')\ da$ is the convolution of a function $Q^{\pi_{I,\ell}}(s,a)$ of $a$ which is bounded by $Q_{\max}$, and the probability density of the noise, which satisfies Assumption \ref{noiseass} and is therefore $L_\ell$ TV-Lipschitz. This allows us to apply Proposition \ref{smooth}, having that the integral, as a function of $a$, is $L_\ell Q_{\max}-$LC. This entails that
    $$ \int_A (\pi_E(a'|s)-\pi_I(a'|s))\bigg(\int_A Q^{\pi_{I,\ell}}(s,a) \ell(a-a')\ da\bigg)\ da'\le 2 L_\ell Q_{\max} \wass(\pi_E(\cdot|s),\pi_I(\cdot|s)).$$
    Therefore, we get:
    $$J^{\pi_{E,\ell}}-J^{\pi_{I,\ell}}\le \frac{2 L_\ell Q_{\max}}{1-\gamma}\E_{s\sim d^{\pi_{E,\ell}}}[\wass(\pi_E(\cdot|s),\pi_I(\cdot|s))],$$
    as required.
\end{proof}

\section{Miscellaneous results}
\label{app:misc}

\begin{prop}
    Let $\mu, \nu$ two probability distributions over $\mathbb R^n$. Then,
    $$\mathbb E[\mu]-\mathbb E[\nu] \le \wass(\mu, \nu)$$
\end{prop}
\begin{proof}
    It is sufficient to consider that
    $$\mathbb E[\mu]-\mathbb E[\nu] = \int_{\mathbb R^n} x \de \mu(x) - \int_{\mathbb R^n} x \de \nu(x) =
    \int_{\mathbb R^n} x (\de \mu(x)-\de \nu(x))\le \sup_{\|f\|_L\le 1}
    \int_{\mathbb R^n} f(x) (\de \mu(x)-\de \nu(x))$$
    where the last term is precisely the definition of Wasserstein metric.
\end{proof}

This result is useful to see that the bound that we are able to obtain with noise injection, that contains $\E_{s\sim d^{\pi_E}}[\wass(\pi_E(\cdot|s),\pi_I(\cdot|s))]$, is even better than its "noisy" counterpart $\E_{s\sim d^{\pi_{E}}}[\wass(\pi_{E,\ell}(\cdot|s),\pi_{I,\ell}(\cdot|s))]$ which measures the difference between the noisy policies. This is because, if the non noisy expert/imitator are deterministic, at fixed $s$, $\pi_E(\cdot|s),\pi_I(\cdot|s)$ are Dirac's delta, and so coincide with their mean, while $\pi_{E,\ell}(\cdot|s),\pi_{I,\ell}(\cdot|s)$ are distribution with mean $\pi_E(\cdot|s),\pi_I(\cdot|s)$ respectively. Therefore, we have:

$$\forall s\in \Ss: \qquad \wass(\pi_E(\cdot|s),\pi_I(\cdot|s))\le \wass(\pi_{E,\ell}(\cdot|s),\pi_{I,\ell}(\cdot|s)).$$

\section{Details of the simulation}
\label{app:experiments}

In the numerical simulation, we first trained some expert policy on the OpenAI \texttt{gym} environments described in the main paper, and then we measured their performance applying an i.i.d. Gaussian noise to the actions chosen by the policies. Precisely, we used an $\mathbb N(0,\sigma I)$ noise with
$$\sigma = 0.0,\ 0.033,\ \dots, 0.266,\ 0.3.$$
for a total of ten equispaced values in $[0,0.3]$.

In order to evaluate the performance of the policies, we made the mean over $40$ episodes for each value of $\sigma$, and then repeated for $20$ different random seeds. Once collected the data, to generate the plot \ref{noisy_effect} we reported the sample mean over the $20$ seeds and a $95\%$ confidence interval obtained with the bootstrapping method.

\subsection{Expert policy training}

In order to use expert policies that are as standard as possible, we used pretrained agents from the gihub repository \texttt{https://github.com/DLR-RM/rl-baselines3-zoo}.

There, it is possible to find all the details about the hyperparameters.

\subsection{Random number generation}

The two possible source of randomness are given by the random generation of the noise and by the intrinsic noise in the environments. When doing the simulations, we simply set both the random seed of \texttt{numpy} and the one of \texttt{gym} to $i=0,1,2,...19$ for each of the twenty experiments.

\subsection{Device and computational power required}

We ran the simulation on \texttt{Python 3.7.9} on a device with CPU $1.80$ GHz and $8.0$ GB of RAM.

\end{document}